\documentclass[runningheads]{llncs}
\usepackage{graphicx}
\usepackage{amsmath}
\usepackage{amssymb}
\usepackage{mathtools}
\usepackage{dirtytalk}
\usepackage[linesnumbered,lined,ruled,commentsnumbered]{algorithm2e}

\usepackage{todonotes}
\usepackage[normalem]{ulem}

\newcommand{\mpara}[1]{\medskip\noindent{\bf #1}}

\newcommand{\lymph}{\textsc{Lymph} }
\newcommand{\diabetes}{\textsc{Diabetes} }
\newcommand{\wdbc}{\textsc{Wdbc} }
\newcommand{\heart}{\textsc{Heart} }
\newcommand{\adult}{\textsc{Adult} }

\newtheorem{RQ}{RQ}

\begin{document}
\title{Achieving differential privacy for $k$-nearest neighbors based outlier detection by data partitioning}


%
%
 \author{Jens Rauch\inst{1} \and
 Iyiola E. Olatunji\inst{2} \and
 Megha Khosla\inst{2}}
 \authorrunning{J. Rauch et al.}
%
 \institute{Health Informatics Research Group, University of Applied Sciences Osnabr\"uck, Germany
 \email{j.rauch@hs-osnabrueck.de} \and
 L3S Research Center, Leibniz University Hannover, Germany\\
 \email{\{iyiola,khosla\}@l3s.de}}
\maketitle              
\begin{abstract}
When applying outlier detection in settings where data is sensitive, mechanisms which guarantee the privacy of the underlying data are needed. The $k$-nearest neighbors ($k$-NN) algorithm is a simple and one of the most effective methods for outlier detection. So far, there have been no attempts made to develop a differentially private ($\epsilon$-DP) approach for $k$-NN based outlier detection. Existing approaches often relax the notion of $\epsilon$-DP and employ other methods than $k$-NN. We propose a method for $k$-NN based outlier detection by separating the procedure into a fitting step on reference inlier data and then apply the outlier classifier to new data. We achieve $\epsilon$-DP for both the fitting algorithm and the outlier classifier with respect to the reference data by partitioning the dataset into a uniform grid, which yields low global sensitivity. Our approach yields nearly optimal performance on
real-world data with varying dimensions when compared to the non-private versions of $k$-NN.

\end{abstract}

\section{Introduction}
\label{intro}
Outliers are observations in a dataset, which deviate considerably from the remainder of the data and might therefore be indicative of a different data generating mechanism \cite{williams2002}. Detection of outliers has important applications in medicine, finance and administrative monitoring, since it allows to identify anomalous or suspicious data for later scrutiny. However, in many use cases such as medical screenings, the data is inherently sensitive. This calls for privacy preserving mechanisms that protect the data of individuals from being revealed when releasing a data analytic model that relies on these data, as is the case for outlier detection models. One powerful and widely accepted approach for releasing data statistics or the trained model while concealing the information about individuals in the employed dataset is \emph{differential privacy} (DP)  \cite{dwork2008}. We specifically focus on $\epsilon$-DP in this work where $\epsilon$ quantifies the privacy guarantee.

Under $\epsilon$-DP, one adds a sufficient level of noise during analysis so that it is impossible to infer the presence or absence of single observations in the underlying dataset from the final model. Perturbing model parameters by adding noise evidently leads to a decreased model performance, which depends on the algorithm for model training, as well as the specific dataset. It is therefore of vital importance to study the impact on model precision when applying $\epsilon$-DP to outlier detection algorithms.

Despite the extensive and growing body of methods in both outlier detection and differentially private data mining, there is only scant literature on how to perform outlier detection with privacy guarantees. In this paper, we therefore combine two pertinent approaches from both domains, namely $k$-nearest neighbor ($k$-NN) based detection of outliers and $\epsilon$-DP. Such a conjoint approach has not been been examined before.

The existing works, which have proposed solutions for privacy preserving outlier detection either relax the concept of $\epsilon$-DP \cite{asif2020guide,bohler2017privacy,lui2015outlier} or use simplified definitions of outliers \cite{okada2015a,bittner2018using,kearns2016private} in order to obtain reasonable detection performance. It is common to all these previous approaches that they make no conceptual distinction between the training of an outlier classifier and identifying outliers in a dataset. With our approach, instead, we could directly build upon the original definitions of $\epsilon$-DP and $k$-NN based outlier detection. We could achieve this by separating the procedure into building a $k$-NN classifier, which precedes the identification of outliers. The algorithm fulfils $\epsilon$-DP and outputs a $k$-NN based outlier classifier, which can be interactively queried to determine outliers in new data. The data used for the fitting procedure remains private as guaranteed by $\epsilon$-DP.
Moreover, this approach is computationally more efficient than the approaches, which do not separate into training and outlier detection, when perpetually applied to new incoming data \cite{hamlet2017}.

While there are a number of advanced outlier detection methods, it has been shown recently, that in many cases simple $k$-NN queries do perform as well as the sophisticated approaches \cite{campos2016}. When using $k$-NN, an observation's distance to its nearest neighbors is a measure of its \say{outlierness}. In fact, a number of the advanced methods based on density estimates, connectivity or angular distance make use of $k$-NN queries, such as Local Outlier Factors (LOF) or Kernel Density Estimation Outlier Score (KDEOS) \cite{tang2017,schubert2014,kriegel2010}.
Being a simple and nonetheless very effective method, $k$-NN is an ideal candidate for outlier detection under $\epsilon$-DP. To the best of our knowledge, this is the first study which examines an approach of differentially private $k$-NN based outlier detection.

\mpara{Our Contributions.} We propose a novel differential private algorithm for outlier detection. The $\epsilon$-DP mechanism works by building a uniform grid on a reference dataset, partitioning it into cells and perturbing the count of elements in the cells. The resulting model of reference cell counts can then be queried against with new data of interest, to determine whether any of the new data points should be considered outliers. Our approach guarantees $\epsilon$-DP for the reference data and allows for online querying of new incoming data \cite{hamlet2017}. We experimentally evaluate our approach on five benchmark datasets in comparison to four non-private baselines. To summarize, we make the following contributions.

\begin{enumerate}
\item We develop a novel yet simple  algorithm for outlier detection with provable differential privacy gurantees. Our approach is scalable for high dimensional datasets as well as continuous domains.
\item We showcase the effectiveness of our approach by comparing our approach with several of the non-private baselines. Our approach yields close to optimal performance on  real-world datasets of varying dimensionality, when compared to the baselines.
\end{enumerate}
We have published the anonymized version of our code.\footnote{\url{https://anonymous.4open.science/r/fcf95211-3361-440a-986a-966bd13b1ede/}}

\section{Related Work}
\label{related}
Okada \textit{et al.} \cite{okada2015a} presented differentially private queries for fixed distance-based outlier analysis. However, their work is targeted at outlier characteristics, rather than outlier detection. It uses a fixed-distance definition of outliers, whereas our work is based on $k$-NN based outlier definition. Moreover, they set out from the relaxed definitions of $\epsilon-\delta$-DP and smooth sensitivity instead of the stricter global sensitivity, which we use. The two types of queries they provide include (i) count queries which returns the number of detected outliers in a given subspace and (ii) discovery of top-subspaces containing a large number of outliers. 

In another line of work, Lui and Pass \cite{lui2015outlier} introduced a new  privacy notion for outliers as a generalization of $\epsilon$-DP with the goal of granting higher privacy protection to outliers than inliers. Their definition, also called tailored differential privacy (TDP) measures the privacy of an individual by how much of an outlier the individual is. Therefore, their work is focused on releasing histograms where outliers are provided more privacy. Moreover, their goal is to define privacy from the perspective of outlier rather than differentially private outlier detection.

For differentially private outlier analysis, several relaxation for $\epsilon$-DP have been proposed. They include anomaly-restricted DP \cite{bittner2018using}, protected DP \cite{kearns2016private}, relaxed sensitivity \cite{bohler2017privacy}, and sensitive privacy \cite{asif2019accurately}. 
All $\epsilon$-DP relaxation except for \cite{asif2019accurately} assume that the outlier model is data-independent. 
In fact, Bittner \textit{et al.} \cite{bittner2018using} assumes that the database only contains one outlier while Kearns \textit{et al.} \cite{kearns2016private} is specific to anomalies searching in social networks. B{\"o}hler \textit{et al.} \cite{bohler2017privacy} provided a relaxation of global sensitivity by assuming that there is a separating boundary between outlier records from the inlier records. Moreover, their data perturbation does not require a privacy budget distributed over a series of queries. 
However, all the relaxed DP methods for outlier detection except for sensitive privacy \cite{asif2019accurately} are not applicable to our work since our outlier model is data-dependent (distance based). 
 
Practically, outlier models are data-dependent where the outlyingness of a record is defined by its distance to the other records in the database. Sensitive privacy \cite{asif2019accurately} generalizes the above notions of privacy and provides a formal privacy framework of $\epsilon$-DP for anomaly detection. This ensures that the outlier identification mechanism is unaffected by the presence or the absence of the individual's record in the database. Their approach first defines the the notion of sensitive record and constructs an appropriate sensitive neighborhood graph that can be used for outlier detection. However, sensitive privacy cannot be directly applied to our problem. 

Another approach of satisfying DP for outlier detection is by addition of noise to the aggregate statistics. Fan and Xiong \cite{fan2013differentially} proposed anomaly detection framework from continual aggregates of user statistics with $\epsilon$-DP guarantee by adding Laplace noise. A filtering algorithm which generate posterior estimates takes the perturbed aggregate and if it passes the sensitivity analysis, it is released else it is corrected, then released. This increases the accuracy of released aggregates. However, their method only releases a one-dimensional time series with differential privacy and outlier detection is applied to the released data as a post process.

Recently, differentially private $k$-NN algorithm has been proposed by Gursoy \textit{et al.} \cite{gursoy2017differentially} which works by first converting $k$-NN classifiers to private radius neighbors ($r-N$) classifiers. The $r-N$ classifier utilizes majority vote among neighbors within a given radius. In order to add noise to satisfy $\epsilon$-DP, sensitivity analysis is performed over a region overlap graph that determines the overlaps among the radii $r$ of the test instances. Hence, given a test instance, the task is to find an accurate $r$ without leaking the distances among the training instances or data distribution. However, the choice of $r$ needs to be carefully chosen. For example, adding Laplacian noise with variance greater than 1 to a significantly low radius, say $r$=0.01 would completely destroy its accuracy and lead to extremely inaccurate results. Another drawback of their approach is that the data owner must be online and available to perform classification on a querier's demand since the privacy budget is distributed over a series of queries. A non-interactive algorithm was proposed to alleviate the drawback but fails when the dimensionality of data exceedingly high or $\epsilon$ is very small. However, our approach is scalable for high dimensional data. Moreover, their approach does not focus on designing a differential private outlier detection algorithm.

\section{Problem definition and basics of differential privacy}
\label{sec:basics}
After stating the problem definition, we will briefly revisit the key concepts of $k$-NN based outlier detection and $\epsilon$-DP.

\subsection{Problem definition}
\label{sec:problem}
From the outlier analysis perspective, there are two classes of samples within a dataset. The first class consists of inliers which form the majority of data in the dataset. The inliers are samples that stem from a common, arbitrary generating distribution. The second class are the remaining samples which are considered to be the outliers. The outliers follow (potentially several) different generating distributions, distinct from the inlier distribution. 

While we have no prior knowledge regarding the generating distributions, we possess a reference dataset $X$, which consists of inliers only and can be thought of as a training set. The \emph{task} is to give a $\epsilon$-DP algorithm that yields an outlier classifier (a scoring function) based on $X$. Suppose we are given a set of new data points $Y$, the outlier classifier should return an outlier score $s_y$ for each $y \in Y$  which reflects a degree of outlierness. In the following we describe the $k$-NN based approaches for computing outlier scores.

\subsection{$k$-NN outlier detection}
\label{sec:od}
$k$-NN based outlier detection works by considering those data points as outliers, which have a large distance to their neighbors in the reference dataset. There are two variants. The basic $k$-NN algorithm considers the distance to the $k$th nearest neighbor as an outlier score, while the weighted $k$-NN algorithm considers the total distance to all $k$ neighbors.

The rationale behind using $k$-NN for outlier detection is that inliers are assumed to lie in regions with relatively high density, whereas outliers are to be found in the less densely populated regions. Points in high density regions have small distances to their neighbors compared to points in low density regions. The parameter $k$ can be made larger to take into account that outliers themselves might appear as higher density clusters, albeit smaller in number. 

In the following, let $X$ be the reference dataset and $y$ a point in the test set $Y$, for which we want to obtain an outlier score $s_y$.

\subsubsection{Basic $k$-NN.}
\label{sec:knn}
Formally, for a given data point $y$ the basic $k$-NN algorithm \cite{ramaswamy2000} computes its outlier score $s_y$ as the euclidean distance to its $k$th nearest neighbor $x_k$ in $X$,
\begin{equation}
s_y = ||x_k - y||_2.
\label{eq:knn}
\end{equation}

\subsubsection{Weighted $k$-NN.}
In the weighted $k$-NN algorithm \cite{angiulli2002}, instead of only considering the distance to the $k$th nearest neighbor, the outlier score is computed as the sum of the distances to \emph{all} $k$ nearest neighbors $x_1, \dots, x_k$ of $y$ in $X$,
\begin{equation}s_y = \sum_{i=1}^k ||x_i - y||_2.
\label{eq:wknn}
\end{equation}
This is called \emph{weighted} $k$-NN, since the neighbors of $y$ are weighted by their distance to $y$. In the literature, it is also known as aggregate $k$-NN \cite{schubert2014}, which more appropriately reflects the summation of neighbor distances. However, we stick to the term weighted, because our approach will in fact lead to weights, when applied to this algorithm.

\subsection{Differential Privacy}
The definition of differential privacy is based on neighboring datasets, i.\,e. datasets having the same number of elements, which differ by exactly one element. 
\begin{definition}[Differential Privacy]
  An algorithm $\mathcal{A}$ is said to be ($\epsilon$-)differ\-entially private ($\epsilon$-DP), when for all neighboring datasets $D$, $D'$ and for all possible outcomes $O$ of $\mathcal{A}$, 
\begin{equation}
Pr\,[\mathcal{A}(D) = O] \leq e^\epsilon \cdot Pr\,[\mathcal{A}(D')=O].
\end{equation}
\end{definition}

The parameter $\epsilon$ is the privacy budget. Intuitively, for small $\epsilon$ a differentially private algorithm $\mathcal{A}$ will output the same result for both datasets with high probability. 

The Laplace Mechanism \cite{dwork2008} can be used to design differentially private algorithms. It relies on the global sensitivity of a query function $f: D \rightarrow \mathbb{R}^m$.

\begin{definition}[Global sensitivity]
  The global sensitivity $\Delta f$ of a function $f$ is defined as the maximal worst case change in $f$ when applied to two neighboring datasets $D$, $D'$,
\begin{equation}
\Delta f := \underset{D,\,D'}{\max}||f(D) - f(D')||.
\end{equation}
\label{eq:sens}
\end{definition}
A differentially private algorithm can be constructed from a query function $f: D \rightarrow \mathbb{R}^m$ with the Laplace Mechanism as follows \cite{dwork2008}. Draw i.i.d random variables $\eta_i$, $i=1,\dots,m$ from the Laplace distribution $Lap(0,\epsilon^{-1}\Delta f)$ and add these to the output of $f$,  
\begin{equation}
\mathcal{A}(f,\,D) = f(D) + (\eta_1 ,\dots, \eta_m).
\end{equation}
The algorithm $\mathcal{A}$ is now a $\epsilon$-DP version of query $f$.
Certain privacy guarantees are also ensured under composition and post-processing, as is stated in the following result \cite{mcsherry2009}.

\begin{theorem}[Composition and Post-Processing]
Let $D$ be any dataset, $\mathcal{A}_1$ and $\mathcal{A}_2$ be two algorithms that satisfy $\epsilon_1$-DP and $\epsilon_2$-DP, respectively. Then the following properties hold for $\mathcal{A}_1$ and $\mathcal{A}_2$:
\begin{enumerate}
    \item Releasing the output of $\mathcal{A}_1(D)$ and $\mathcal{A}_2(D)$ satisfies $(\epsilon_1 + \epsilon_2)$-DP \emph{(Sequential composition)}.
    \item Given another \emph{disjoint} dataset $D'$, releasing  $\mathcal{A}_1(D_1)$ and   $\mathcal{A}_2(D_2)$ satisfies $\max(\epsilon_1, \epsilon_2)$-DP \emph{(Parallel Composition)}.
    \item Any post-processed output of $\mathcal{A}_1(D)$ still satisfies $\epsilon_1$-DP \emph{(Post-processing immunity)}.
\end{enumerate}
\label{eq:comp}
\end{theorem}

\section{Our approach}
\label{sec:approach}
\begin{center}
\begin{figure}
\centering
\begin{tikzpicture}[]
\begin{scope}[black!40]
 \foreach \X in {0,...,2}
 {\foreach \Y in {0,...,2}
 {\draw (\X,\Y,0) -- ++(0,0,2);
 \draw (\X,0,\Y) -- ++(0,2,0);
 \draw (0,\X,\Y) -- ++(2,0,0);}}

\node[draw,circle,inner sep=.5pt,fill] at (-.1, .2, .13) {};
\node[draw,circle,inner sep=.5pt,fill] at (-.4, .15, .13) {};
\node[draw,circle,inner sep=.5pt,fill] at (.1, .1, .1) {};
\node[draw,circle,inner sep=.5pt,fill] at (.3, -.5, .4) {};
\node[draw,circle,inner sep=.5pt,fill] at (-.5, -.6, -.44) {};
\node[draw,circle,inner sep=.5pt,fill] at (-.3, -.34, -.44) {};

\node[draw,circle,inner sep=.5pt,fill] at (-.1, 1.2, .13) {};
\node[draw,circle,inner sep=.5pt,fill] at (-.2, 1.115, .13) {};
\node[draw,circle,inner sep=.5pt,fill] at (.1, 1.2, .1) {};
\node[draw,circle,inner sep=.5pt,fill] at (.3, .5, .4) {};
\node[draw,circle,inner sep=.5pt,fill] at (-.5, .6, -.44) {};
\node[draw,circle,inner sep=.5pt,fill] at (-.3, .74, -.44) {};
\end{scope}
\begin{scope}[red]
\node[draw,circle,inner sep=.5pt,fill] at (1.1, 1.1, 0) {};
\end{scope}
\node at (.5, -1.5, 0) {(a)};
\end{tikzpicture} 
\hspace{1cm}
\begin{tikzpicture}[]
\begin{scope}[black!40]
 \foreach \X in {0,...,2}
 {\foreach \Y in {0,...,2}
 {\draw (\X,\Y,0) -- ++(0,0,2);
 \draw (\X,0,\Y) -- ++(0,2,0);
 \draw (0,\X,\Y) -- ++(2,0,0);}}

\end{scope}
\begin{scope}[blue]
\node[draw,circle,inner sep=.7pt,fill] at (-.25, -.2, -.25) {};
\node[draw,circle,inner sep=.7pt,fill] at (-.25, .8, -.25) {};
\end{scope}
\begin{scope}[red]
\node[draw,circle,inner sep=.5pt,fill] at (1.1, 1.1, 0) {};
\end{scope}
\node at (.5, -1.5, 0) {(b)};
\end{tikzpicture} 
\caption{Illustration of the grid modification for $b=2$, $d=3$. The sub-cubes depict the data partition grid. (a) The reference data $X$ in grey and a test point $y$ in red. Two cells contain six data points each. (b) Centroids of the grid cells with weight $w_i=6$ for both populated cells and $w_i=0$ for the other cells.}
\label{fig:lattice-2}
\end{figure}
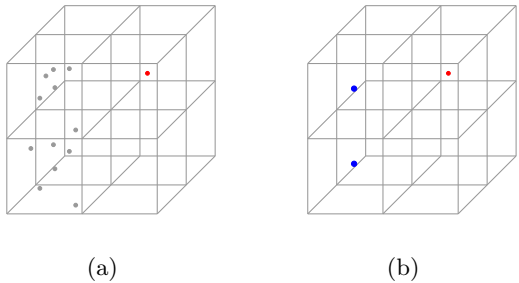
\end{center}

Our differentially private versions of $k$-NN and weighted $k$-NN (w$k$-NN) rely on data partitioning in order to allow querying the reference dataset $X$ with a low global sensitivity. Our approach has been inspired by Su \textit{et al.} \cite{su2016}, who used a similar partitioning procedure for $\epsilon$-DP $k$-means clustering. We partition the dataset by choosing a number $b$ and splitting each dimension into $b$ equally spaced intervals. This results in a uniform data grid on top of the reference dataset, so that data points are being assigned to grid cells. In the following, we refer to this procedure as \emph{grid} modification of (w)$k$-NN. 

Instead of considering the exact location of points in the reference set $X$ to determine an outlier score for $y$, our grid modifications treat the points of each cell as if they were all located at the cell's centroid. Essentially, our algorithm then only considers the centroids, weighted by the number of points in the respective cell, to determine the outlier scores by Equations \ref{eq:knn} and \ref{eq:wknn}. Our procedure is illustrated in Figure \ref{fig:lattice-2}, where two data partitions contain six points each. To determine the outlier score of the red point with respect to the reference data, instead of considering the raw location of each $x\in X$, we treat the centroids of all cells as data points with weights according to the number of elements in the respective cell as shown in Figure \ref{fig:lattice-2}(b). The formal approach is given in Section \ref{sec:grid}. These grid modification will subsequently be made $\epsilon-DP$ (Section \ref{sec:dpgrid}). For better comparison, we will analyze both the non-private and $\epsilon-DP$ versions of the grid modification.

\subsection{Preprocessing}
\label{sec:pre}
Before applying outlier detection methods, datasets should be normalized and centered, since this can improve performance drastically \cite{kandanaarachchi2020}.
In particular, for a given $d$-dimensional dataset $X$  with $x_{ij}$ representing the $j$th dimension/feature value of the $i$th element, we first scale each dimension $j$ to the unit interval by $\alpha_j$ (maximum absolute value), i.\,e.
\begin{equation}
    \alpha_j = \underset{i=1,\dots,N}{\max}|x_{ij}|.
\end{equation}
We then perform a linear mapping from the resulting interval $[-1,1]$ onto the unit interval. Next we center each dimension by subtracting the respective mean value $\bar{x}_j$. We obtain a centered unit hypercube. The scaling and centering parameters $\alpha_j$ and $\bar{x_j}$ are retained for later processing of the test data. Coordinates of data points in the test data that lie outside the unit interval after transformation are mapped to $0$ or $1$, depending on whether they are negative or greater than 1, respectively. 

\subsection{Grid modification of $k$-NN and weighted $k$-NN }
\label{sec:grid}
For our approach, we choose a grid parameter $b$ and construct a regular grid on the dataset by splitting each dimension into $b$ equidistant intervals, resulting in a grid partition with a total of $d^b$ cells, each being a hypercube with edge length $\frac{1}{b}$. 
Let $q_f(c)$ denote the number of data points in each cell $c$ (count query).
To answer a $k$-NN query, a trusted server with access to the original dataset $X$, will perform Algorithm \ref{alg:dpknn}. 

Our algorithm first looks up the grid cell $c_y$ in the reference dataset, in which a data point of interest $y$ resides (line \ref{alg:lookup}). It will then traverse all neighboring cells $C_y$ of data point $y$ in the order of their respective distance to $y$ (where distance is measured from the centroid of each neighboring cell $c \in C_y$; lines \ref{alg:traverseforin}--\ref{alg:traverseforend}). The algorithm terminates either when the cumulative number of elements in the traversed cells exceeds the threshold value $k$ (line \ref{alg:terminatek}), or when a given maximal neighbor depth $\delta_{\max}$ has been reached (line \ref{alg:traverseforend}). In case of our grid modification of basic $k$-NN, it will return the distance from point $y$ to the last visited neighbor cell (line \ref{alg:scoreknn}). In case of our grid modification of w$k$-NN, it will return the weighted sum of all distances of the visited cells; the weights beeing the count queries $q_f(c)$ (line \ref{alg:scorewknn}).

In the next subsection we will formally describe the mechanism, by which our grid modifications are made $\epsilon$-DP.

\IncMargin{1em}
\begin{algorithm}[h]
\DontPrintSemicolon
\SetKwInOut{Input}{Input}\SetKwInOut{Output}{Output}
\Input{preprocessed reference dataset $X\in \mathbb{R}^{n\times d}$, a new data point $y \in \mathbb{R}^{d}$, partitioning parameter $b\in \mathbb{N}$, parameter $k\in \mathbb{N}$, maximal neighbour depth $\delta_{\max}$, $weighted \in \{0,1\}$}
\Output{outlier score $s_y$}
\Begin{
$c_y \longleftarrow$ grid cell, in which $y$ resides\;\label{alg:lookup}
$C_y \longleftarrow$ neighboring cells of $c_y$ with $||\bar{c} - \bar{c}_y||_1 \leq \delta_{\max}, \forall c \in C_y$, where $\bar{c}$ denotes the centroid of a cell \;
\For{$c\in C_y$}{
\tcp{Collect distances to point $y$ to neighbor cells for sorting}
$dist_{y}(c) \longleftarrow ||y - \bar{c}||_1$ \;
}
$\widetilde{C_y} \longleftarrow C$ sorted by $dist_{y}(c_i)$ in ascending order\; 
$Q \longleftarrow 0$
$s_y  \longleftarrow 0$ \;
\For{$c\in \widetilde{C_y}$ \label{alg:traverseforin}}{
\tcp{$Q$ will keep track of how many points (cell count) we have visited}
$Q \longleftarrow Q + q_f(c)$\; \label{alg:Q}
$dist_c \longleftarrow ||\bar{c} - \bar{c}_y||_1$\;
\eIf{weighted}{
$s_y  \longleftarrow s_y + q_f(c) \cdot dist_c$ \tcp*[r]{Sum weighted distances}
\label{alg:scorewknn}}{$s_y  \longleftarrow dist_c$ \tcp*[r]{Keep last distance only} \label{alg:scoreknn}} 
\If{$Q \geq k$}{break \tcp*[r]{Terminate when threshold $k$ exceeded} \label{alg:terminatek}}
\label{alg:traverseforend}}
\Return $s_y$
}
\caption{Grid modification of (w)$k$-NN outlier detection \label{alg:dpknn}}
\end{algorithm}
\DecMargin{1em}

\subsection{Privacy mechanism for our approach}
\label{sec:dpgrid}
The count queries $q_f$ in our grid modifications of (w)$k$-NN can be perturbed according to the Laplace-Mechanism to fulfil $\epsilon$-DP. To this end, we first examine the global sensitivity of the count query function $q_f$, which returns the number of elements in a given cell. The result is given in the following Lemma.

\begin{lemma}
\label{lemma}
Let $D$ be any dataset, $C_D$ a partition of the domain of $D$ into cells and $q_f: C_D \rightarrow \mathbb{N}_0$ the count query, which returns the number of elements of $D$ $c\in C_D$. Then, the global sensitivity of the count query function  \begin{equation}
    \Delta q_f = 1.
\end{equation}
\end{lemma}
\begin{proof}
Let $D$, $D'$ be two neighboring datasets, then there exists an $x \in D$, which is replaced by $x' \in D'$, so that $D\setminus{\{x\}}=D'\setminus{\{x'\}}$. Therefore, for any $c \in C_D$ we have a corresponding $c' \in C_{D'}$, with  $c\setminus{\{x\}}=c'\setminus{\{x'\}}$ and hence $|q_f(c)-q_f(c')|=\big||c| - |c'|\big| = |\mathbf{1}_{\{x\}\in c} - \mathbf{1}_{\{x'\}\in c'}|$, which is either $0$ or $1$. \qed
\end{proof}
Thus, we obtain a differentially private count query $q^*_f(c)$ by adding Laplace noise $\eta_c \sim Lap\left(\frac{1}{\epsilon}\right)$,
\begin{equation}
q_f^*(c) := q_f(c) + \eta_c.
\end{equation}
Note that $q_f^*$ can take negative values, which however is desired, since otherwise the  distribution of cell counts would be skewed towards positive values \cite{su2016}. The noise variables $\eta_c$ will be realised only once per cell $c$, to ensure privacy for repeated cell querying (see Theorem \ref{eq:dpguarantee}), which is indicated by the index $c$.

To answer a $k$-NN outlier detection query under $\epsilon$-DP, a trusted server with access to the original unperturbed dataset will perform the grid modifications of (w)$k$-NN (Algorithm \ref{alg:dpknn}), albeit replacing $q_f$ by the perturbed count query $q^*_f$ in lines \ref{alg:Q} and \ref{alg:scorewknn}. The server will retain all perturbed cell counts $q^*_f(c)$, once realised, and output the retained value, in case a cell is queried repeatedly in order to prevent leakage of the true value. 

\begin{theorem}
Replacing $q_f$ by $q_f^*$ in Algorithm 1 will give an $\epsilon$-DP version of our grid modification, provided that the random noise is realised only once giving $\eta_c$ for each cell $c$ and then being kept constant.
\label{eq:dpguarantee}
\end{theorem}
\begin{proof}

The Laplace Mechanism guarantees that count query $q_f^*$ is $\epsilon$-DP by Lemma \ref{lemma}. Algorithm 1 returns outlier score $s_y$, which is dependent on the distance between cell centroids, $dist_c$  as well as $Q$ (the variable that tracks how many points have been visited; see line 15) and $\delta_{\max}$ (termination conditions). The maximum  depth $\delta_{\max}$ is constant. All distances $dist_c$ are independent of the dataset $D$, but their composition (or last value in case of basic $k$-NN) is dependent on $Q$. However, $Q$ is a parallel composition of $q_f^*$ for disjoint cells and by Theorem \ref{eq:comp} it is $\epsilon$-DP. By the post-processing property in Theorem \ref{eq:comp} the computation of $dist_c$ is then also $\epsilon$-DP. Finally, repeated querying of the same cell will result in the same output, since the noise variable  $\eta_c$ is realised only once for each cell, so that no further information is gained, in case of sequential composition of Algorithm 1 on the same cell ($\epsilon_i=0, \forall i\geq 2$). \qed

\end{proof}

\section{Experiments}
\label{sec:exp}
We designed our experiments to investigate the following research questions:

\begin{RQ}
\label{rq:mod}
 How does our non-private grid modification impact outlier detection performance compared to basic $k$-NN and weighted $k$-NN?
\end{RQ}
\begin{RQ}
\label{rq:compare}
How do our $\epsilon$-DP grid modifications compare to the non-private baselines (w)$k$-NN and our non-private grid modifications?
\end{RQ}
\begin{RQ}
\label{rq:privacybudgeteffect}
What is the effect of the privacy budget $\epsilon$ on the performance of our approach?
\end{RQ}

\subsection{Dataset}
\label{dataset}
We evaluated our approach using five benchmark real-world classification datasets that have been previousy used for evaluating outlier detection algorithms. We obtained all these datasets from the UCI Machine Learning Database \cite{Dua2019}. Statistics of all dataset are shown in Table \ref{tab:1}

\mpara{\lymph} The \lymph dataset represents lymphographical patient data and is divided  according to radiographical examination findings. Only six patients had findings \say{normal} or \say{fibrosis}. These are considered outliers. Three of the 17 measured attributes are continuous.

\mpara{\diabetes} The \diabetes dataset consists of diabetes test findings of Pima Indians and eight medical predictors, all being continuous. The outliers are the subjects, who were diagnosed positive. 

\mpara{\wdbc} The \wdbc dataset describes diagnostic findings for breast cancer. Nuclear characteristics are represented in 30 continuous attributes. Outliers are the malignant findings.

\mpara{\heart} The \heart dataset consists of patients classified as healthy or as having heart related problems. It has 13 continuous predictor variables. 

\mpara{\adult} The \adult is the only non-medical dataset. We included it, since it was used in \cite{okada2015a} for evaluation of a privacy preserving method for outlier detection. It contains seven socio-demographic continuous predictors and an income level class. The outliers are individuals, who had a yearly salary above \$50,000.

\begin{table}
\begin{center}
\caption{Statistics of the datasets used in experiments.}
\label{tab:1}    
\begin{tabular}{ lrrrrr  }
\hline\noalign{\smallskip}
 & \multicolumn{5}{c}{\textsc{Dataset}} \\
  \cline{2-6}\noalign{\smallskip}
 \textsc{Statistic} & \lymph  & \diabetes & \wdbc & \heart & \adult\\
 \hline\noalign{\smallskip}
 Dimension $d$ & 3 & 8 & 30 & 13 & 7 \\
 Size of reference set & 113 & 400 & 285 & 110 & 19,776 \\
 Size of test set &  35  & 140 & 82 & 38 & 5,044\\
 Outliers in test set  & 6  & 40 & 10 & 10 & 100\\

 \hline\noalign{\smallskip}
\end{tabular}
\end{center}
\end{table}

\subsection{Setup}
\label{sec:exp_setup}

Adhering to common practice \cite{schubert2014}, we treated the first $m$ instances of the minority class in each dataset as outliers. The value $m$ was chosen according to previous studies with these datasets \cite{zhang2009,campos2016}. Consequently, the outlier were downsampled in all datasets except the \lymph data.
The remaining data were labeled as inliers. The datasets and their statistics are shown in Table \ref{tab:1}. We only included continuous and dichtomous variables from the datasets.

For each dataset, 80\,\% of the inliers were used as training dataset, which was used by the algorithm as a reference set to identify outliers. The remaining 20\,\% were appended to set of outliers to build the test set. Pre-processing was applied to both datasets as described in Section \ref{sec:pre}. 

We ran our algorithm on each dataset with different values of $k$-NN parameter $k$, grid parameter $b$ and privacy budget $\epsilon$. After selecting appropriate values of $k$, we first varied the grid parameter $b$ in range $b\in[2,10]$, to examine its impact on both our non-private and $\epsilon$-DP grid modifications. From these results, we selected the smallest $b$ that maximized the AUROC of our $\epsilon$-DP grid modifications for each dataset. We used the selected parameters of $k$ and $b$ to evaluate our $\epsilon$-DP grid modifications (DP Grid $k$-NN and DP Grid w$k$-NN) for different values of $\epsilon$ in comparison to the non-private grid modifications and (w)$k$-NN.
The algorithm was run with values of $\epsilon \in \{5, 2.5, 1.25, .6, .3, .15, .075, .035, .015\}$ to evaluate its performance with respect to the privacy budget. 

To account for random variability in the private algorithms' performance due to the Laplace noise, both $\epsilon$-DP versions of Algorithm 1 were applied ten times on each dataset with different initial seeds for the random number generator. We report mean performance and standard deviation over the ten runs.

\begin{figure}
\includegraphics[width=\textwidth]{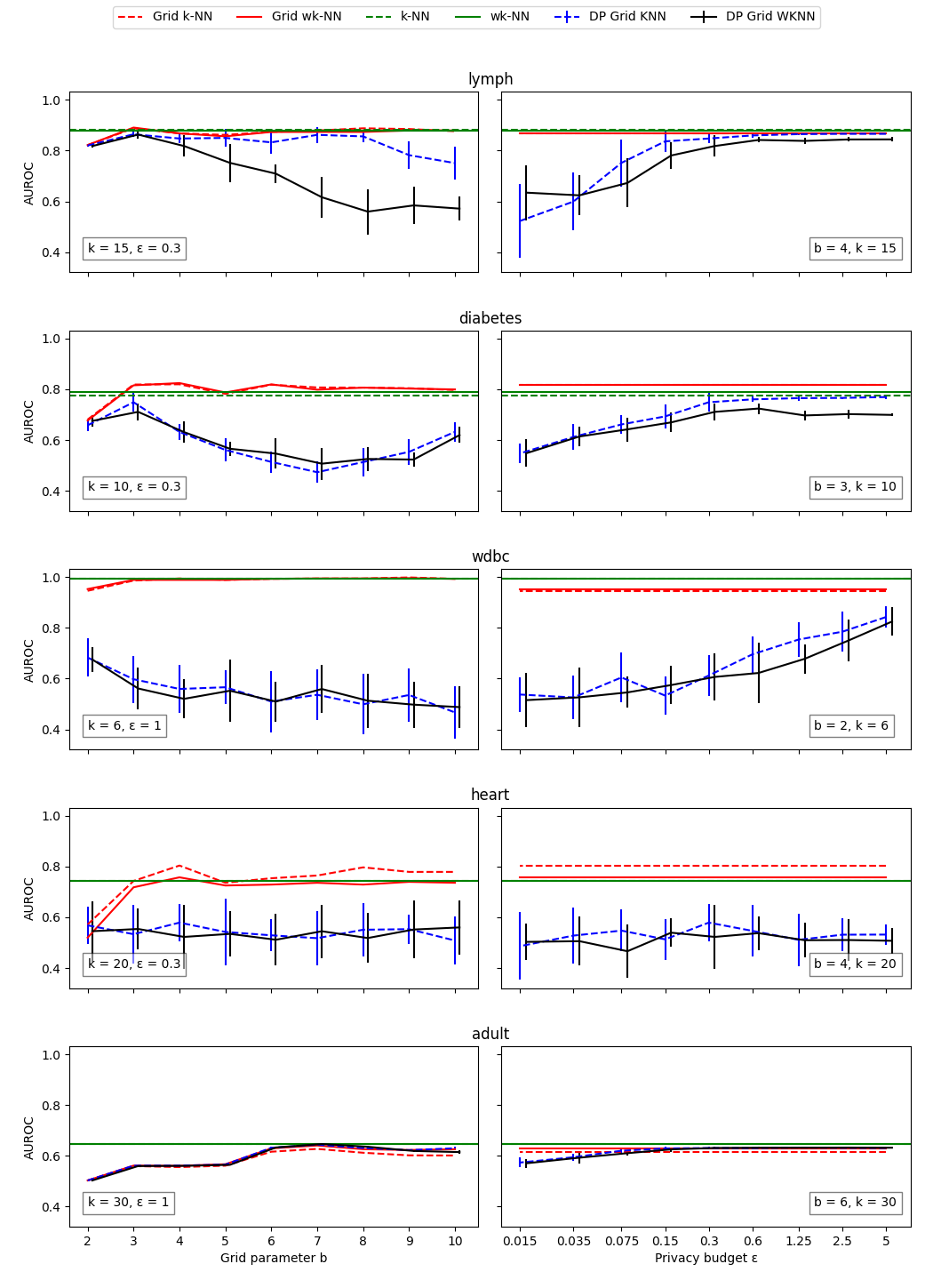}
\caption{Comparison of algorithms with respect to the AUROC for different grid partition parameters $b$ (left) and privacy budgets $\epsilon$ (right). } \label{fig:results1}
\end{figure}

\begin{figure}
\includegraphics[width=\textwidth]{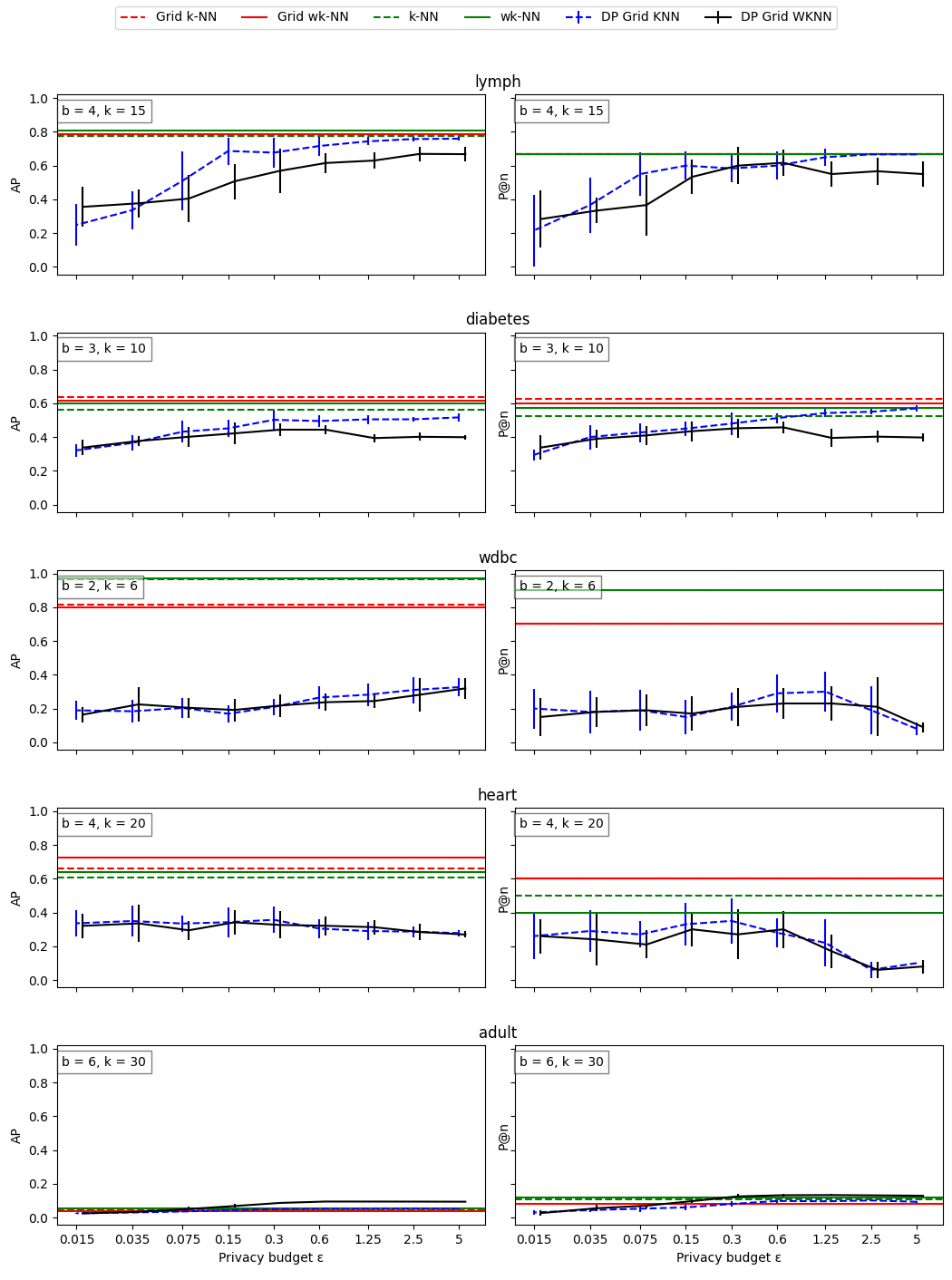}
\caption{Comparison of algorithms with respect to the average precision (AP, left) and the precision at $n$ ($P@n$, right) for different privacy budgets $\epsilon$.} \label{fig:results2}
\end{figure}

\subsection{Results} 
\label{sec:results}
We report outlier detection performance on the five datasets in terms of the AUROC, the average precision (AP) and the precision at $n$ ($P@n$) \cite{campos2016}. We  compare our approach to the four baselines, basic and weighted $k$-NN as well as the non-private grid modifications of these. For $P@n$, we set $n$ to the actual number of outliers in each test set. In case of our private approach, the $\epsilon$-DP grid modifications, mean and standard deviation (error bars) over all ten iterations of the algorithm for different seeds are reported for each dataset. The results are presented in Figures \ref{fig:results1} and \ref{fig:results2}.

\mpara{RQ \ref{rq:mod}.} The unmodified (w)$k$-NN algorithms performed in nearly all cases best, which is expected, because the modified versions introduce information loss (green lines). The \adult dataset was the only dataset, for which none of the algorithms showed acceptable performance. We present the results nonetheless, to document this anomalous case.

We found that our non-private grid modifications showed only a modest decline in performance for parameters $b = 2$ compared to unmodified (w)$k$-NN. 
Notably, in all but the \adult data even small grid parameters $b \geq 3$ yield mostly comparable results and sometimes even outperforms (w)$k$-NN (green vs. red lines in Figure \ref{fig:results1}, left). This might be explained by considering the grid mechanism as some kind of regularizer, which could help to eliminate irrelvant small deviations within the grid partitions that would otherwise impact the exact $k$-NN distances negatively.

\mpara{RQ \ref{rq:compare}}
For our private grid modifications, a decline in outlier detection performance is expected for larger values of $b$: at first, a larger $b$ increases data resolution, since the grid cells become smaller and hence the distance between the original data points and the cell centroids decreases. Nonetheless, the more cells are created, the more error is introduced by the Laplace Mechanism. This error largely outweighs the benefit of a better resolution, as can be seen in the performance decline in the \lymph and \wdbc datasets (blue and black lines in Figure \ref{fig:results1}, left). In the \adult dataset this relationship is reversed. However, this might be explained by the already large proportion of errors in the grid modifications, which seems unaffected by additional Laplace noise.


\mpara{RQ \ref{rq:privacybudgeteffect}}
For all datasets, our basic $\epsilon$-DP grid algorithm performed slightly better than our weighted version (blue vs. black lines). Our approach showed performance nearly comparable to the four non-private baselines for all metrics on the lower dimensional \lymph ($\epsilon=.15$), \diabetes ($\epsilon=.3$) and \adult ($\epsilon=.15$) data. For the higher dimensional \heart data, the $P@n$ of our private approach was nearly as good as (w)$k$-NN for $\epsilon=.3$. Finally, the AUROC for the \wdbc data, which had the most dimensions, converged towards the non-private baselines for $\epsilon=5$.
As expected, a larger privacy budget $\epsilon$ resulted in better performance. A smaller privacy budgets $\epsilon$ evidently introduces a larger amount of noise being added to the cell counts. This is reflected in the results for the \lymph, \diabetes and the AUROC of the \wdbc data, but not for the \heart or \adult data, and not for the AP and $P@n$ of the \wdbc data. This could be explained by the higher dimensionality of both \wdbc and \heart data, and the anomaly in \adult data. A larger number of dimensions results in a larger number of grid cells even for low values of $b$, and therefore additional noise by the Laplace Mechanism.

\section{Conclusion}
\label{sec:conclusion}
This is the first study, which combines $\epsilon$-differential privacy and $k$-NN based outlier detection. We proposed, analysed and evaluated a $\epsilon$-DP grid modification approach of basic and weighted $k$-NN outlier detection, which performs well on five real-world datasets for reasonably small privacy budgets. We found that our non-private grid modifications of both the basic and weighted $k$-NN result in no notable performance loss compared to unmodified (w)$k$-NN. Consequently, the expected impact on outlier detection performance in the private setting can be solely attributed to the privacy guarantees and does not result from data partitioning. The  results underline that the $\epsilon$-DP grid modification is a promising candidate for privacy preserving outlier detection and could give even stronger results, when applied under relaxed DP guarantees. 

\section*{Acknowledgments}
This work is in part funded by the Lower Saxony Ministry of Science and Culture under grant number ZN3491 within the Lower Saxony "Vorab" of the Volkswagen Foundation and supported by the Center for Digital Innovations (ZDIN), and the Federal Ministry of Education and Research (BMBF), Germany under the project LeibnizKILabor (grant number 01DD20003).

\bibliographystyle{splncs04}
\bibliography{dpknnpaper.bib}

\end{document}